\documentclass{article}

\usepackage{graphicx}
\usepackage{booktabs}
\usepackage[english]{babel}
\usepackage{amsfonts, amsmath, amssymb, amsthm}
\usepackage{xcolor}
\usepackage[affil-it]{authblk}
\usepackage{algorithm}
\usepackage{algpseudocode}
\usepackage{bbm}
\usepackage{multicol}
\usepackage{hyperref}

\newtheorem{thm}{Theorem}[section]

\newtheorem{prop}{Proposition}[section]
\newtheorem*{prop*}{Proposition}

\DeclareMathOperator*{\argmax}{arg\,max}
\DeclareMathOperator*{\nmin}{n-\,min}
\DeclareMathOperator*{\onemin}{1-\,min}
\DeclareMathOperator*{\twomin}{2-\,min}
\DeclareMathOperator*{\threemin}{3-\,min}

\title{Optimizing Sensor Network Design for Multiple Coverage} 

\author[1]{Lukas Taus}
\author[2]{Yen-Hsi R. Tsai}
\affil[ ]{Oden Institute for Computational Engineering and Sciences, University of Texas at Austin}
\affil[1]{{l.taus@utexas.edu}}
\affil[2]{{ytsai@math.utexas.edu}}

\date{}

\begin{document}

\maketitle

\begin{abstract}
Sensor placement optimization methods have been studied extensively. They can be applied to a wide range of applications, including surveillance of known environments, optimal locations for 5G towers, and placement of missile defense systems. However, few works explore the robustness and efficiency of the resulting sensor network concerning sensor failure or adversarial attacks. This paper addresses this issue by optimizing for the least number of sensors to achieve multiple coverage of non-simply connected domains by a prescribed number of sensors. We introduce a new objective function for the greedy (next-best-view) algorithm to design efficient and robust sensor networks and derive theoretical
bounds on the network's optimality. 
We further introduce a Deep Learning model to accelerate the algorithm for near real-time computations. The Deep Learning model requires the generation of training examples. Correspondingly, we show that understanding the geometric properties of the training data set provides important insights into the performance and training process of deep learning techniques. Finally, we demonstrate that a simple parallel version of the greedy approach using a simpler objective can be highly competitive.
\end{abstract}

\section{Introduction}
Sensor placement optimization methods are a well-studied field with far-reaching implications for many application areas, including surveillance of known environments using minimal sensors, one of the most actively studied problems in autonomous robotics. Another application that got significant traction recently was the optimal placement of 5G-wireless towers. 5G communication uses 30GHz-300GHz signals, corresponding to wavelengths ranging from 1mm to 10mm. Compared to the size of the obstacles in a larger domain, this is a high-frequency wave propagation regime; thus, the waves can be well approximated by straight lines. In \cite{5GSurvey}, it was also stated that one of the major difficulties in 5G communication is the signal sensitivity due to blockage. 

However, algorithms for this type of sensor coverage problem typically only consider single coverage. A part of the environment is visible if at least a single sensor observes it. This formulation makes the network of sensors volatile to technical failure and adversarial attacks. 

In this paper, we consider the problem of designing a network with the least number of sensors achieving multiple coverage for a bounded, non-simply connected environment. 
The multiple coverage condition specifies that a region is not considered fully visible until at least $k$ sensors observe it. This also provides measurements from different angles, which improves the accuracy for object reconstruction tasks discussed in \cite{jin2003estimation}.
In some sense, one has two competing objectives: on the one handle, one wishes that each point in the free space be covered by a prescribed minimal number of sensors; on the other hand, one wishes to use as few sensors as possible. 

However, the computational complexity of this problem makes computations of optimal solutions challenging. 
Greedy algorithms may efficiently find near-optimal solutions to many optimization problems. 
However, designing a suitable gain/reward function is essential.
Among the main contributions of this paper are (i) the formulation of a novel gain function with a corresponding estimate of the optimality when it's used in the greedy algorithm, (ii) a Deep Learning strategy to accelerate the computation, and (iii) an empirically highly efficient  parallel greedy algorithm for the multiple coverage problem. 
Some representative numerical examples and comparison of the algorithms' performance statistics are presented.


\subsection{Related work}
The problem of optimally placing sensors to achieve visibility in a given environment is closely related to the gallery problem in computational geometry. An upper bound for sensors needed to achieve complete visibility in simple polygonal environments with holes has been derived in \cite{BjorlingSachs1995AnEA,185346}. For general environments, however, the problem is NP-complete in \cite{URRUTIA2000973,1056648,1057165}. Multiple methods for solving this problem have been proposed, including alternating minimization \cite{GoroMinProb} and transformation of the problem into a system of differential equations \cite{kim2016optimal}. These approaches, however, assume that an a priori fixed amount of sensors are placed. In most applications, however, this is not known. For general two-dimensional environments, it was proposed to place sensors along frontiers \cite{landa2006visibility,landa2008visibility,landa2007robotic}. However, in \cite{StarMapExample}, it was shown that this approach is not necessarily optimal. For general three-dimensional environments, the use of level set representations for the computation of visibility was proposed in \cite{cheng2005visibility,tsai2004visibility}. Another approach to handling three-dimensional environments described by a function graph was discussed in \cite{visalg}. Using these representations, ``next-best-view'' algorithms, which maximize visibility information gain, are deployed to generate optimal sequences of sensor locations \cite{ly2018greedy}. Commonly used measures for visibility information gain include the volume of unexplored regions \cite{greedy1,greedy2,greedy3,greedy4} and surface area of frontiers weighted by viewing angle \cite{valente2012information,valente2014information}. The use of deep learning techniques to evaluate these measures more efficiently has been discussed in \cite{explorationsurveilance}. In addition to achieving optimal visibility, the uncertainty quantification of inaccurate depth measurements of LIDAR sensors was discussed in \cite{uncertaintypaper}. To improve the accuracy of these sensors, it was proposed to optimize sensor locations such that multiple sensors observe the same regions in the environment \cite{taus2023efficient}. In the described approach, however, no theoretical bound for the efficiency of the algorithm has been proven. In this paper, we propose using a novel utility function, which we call ``gain'', to achieve multiple coverage for which we prove theoretical guarantees.

\section{The problem definition}
A given environment consists of areas occupied by an obstacle and free space. We denote this as
$$\Omega = \Omega_\text{obs} \cup \Omega_{\text{free}}.$$
We will consider the environment in $\mathbb{R}^2$ and $\mathbb{R}^3$.

\begin{figure}[H]
    \centering
    \includegraphics[width = 0.49\linewidth]{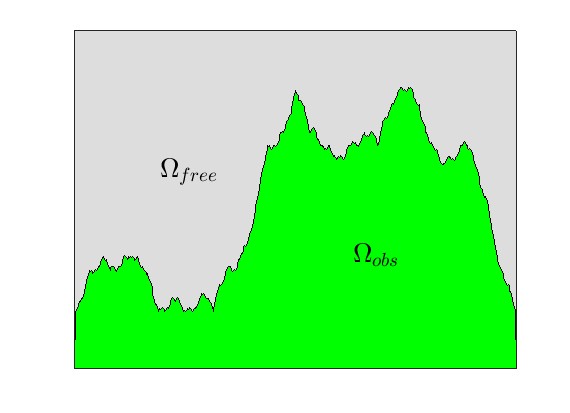}
    \includegraphics[width = 0.49\linewidth]{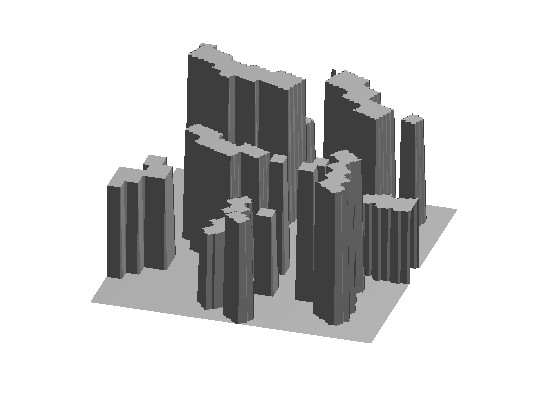}
    \caption{Representations of environments in $\mathbb{R}^2$ and $\mathbb{R}^3.$}
    \label{fig:obstacle}
\end{figure}

Figure~\ref{fig:obstacle} shows representations of different types of environments. The left figure shows an environment in $\mathbb{R}^2$ where the green area is the part of the environment occluded by an obstacle. Similarly, the right figure shows the surface of the obstacles of an environment in 3 dimensions. \\

We define the visibility relation using the line-of-sight principle which means that a point $\mathbf{x} \in \Omega_\text{free}$ is visible to another point $\mathbf{y} \in \Omega_\text{free}$ if 
$$\mathbf{x} \overset{\text{vis}}{\sim} \mathbf{y} \iff  \forall t \in [0,1] \text{: } t\mathbf{x} + (1-t)\mathbf{y} \in \Omega_\text{free}.$$
This equivalence relation is then used to define the visibility function
\begin{equation}
    \phi_{\mathbf{x}}(\mathbf{y}) = \begin{cases}
    1, & \text{if } \mathbf{x} \overset{\text{vis}}{\sim} \mathbf{y},\\
    0, & \text{else.}
\end{cases}
\end{equation}
$\phi_{\mathbf{x}}$ is used to describe the observed area when placing a sensor $\mathbf{x} \in \Omega_\text{free}$. 

\begin{figure} 
    \centering
    \includegraphics[width = 0.5\linewidth]{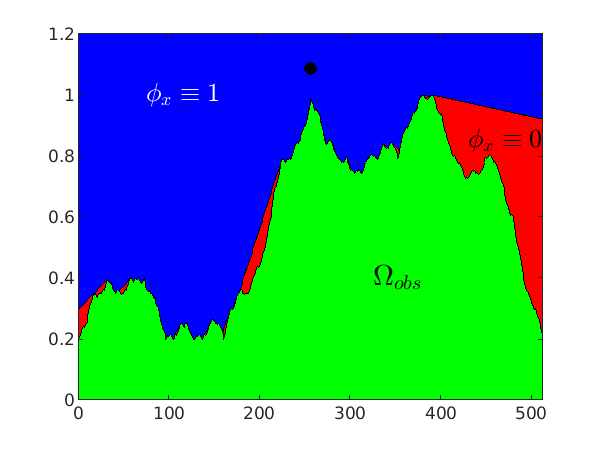}
    \caption{Illustration of $\phi_{\mathbf{x}}(\mathbf{y})$ in a 2 dimensional environment.}
    \label{fig:phi_x_2D}
\end{figure}
We illustrate this concept in Figure~\ref{fig:phi_x_2D}, in which
a sensor $x$ is placed at the location of the black dot. This results in the blue area becoming visible and therefore $\phi_{\mathbf{x}} \equiv 1$ in this region. The red area, however, is not visible to sensor $x$ and therefore $\phi_{\mathbf{x}} \equiv 0$. Note also that $\Omega_{\text{obs}}$ can never be observed by any sensor and therefore also $\phi_{\mathbf{x}} \equiv 0$ in $\Omega_{\text{obs}}$. 


To describe the multiple coverage constraint, however, we must also track areas observed by multiple sensors. We use the order-of-visibility function, which is defined as 
\begin{equation}
    \mathcal{O}_\text{vis}(\mathbf{y}; P) = \sum_{x \in P} \phi_{\mathbf{x}}(\mathbf{y}),
\end{equation}
where $P$ is a set of sensor locations. This function returns the number of sensors in the set $P$ which observe point $\mathbf{y} \in \Omega_\text{free}$ up to a fixed amount $k$. 

\begin{figure}
    \centering
    \includegraphics[width = 0.5\linewidth]{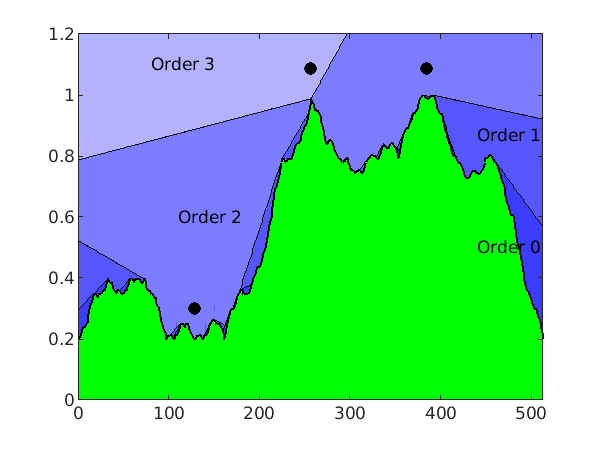}
    \caption{Example of $\mathcal{O}_\text{vis}$.}
    \label{fig:OvisK}
\end{figure}

Figure~\ref{fig:OvisK} shows the order of visibility in an environment where three sensors have been placed. The sensors are again marked using black dots. The regions of visibility overlap now, producing different orders of visibility. The lighter the color blue in Figure~\ref{fig:OvisK} the higher order of visibility we observe. However, there are still areas that are not observed by any sensor. We denote this area and $\Omega_\text{obs}$ as visibility of order 0. With this, we can then precisely describe the multiple coverage problem.\\

\fbox{\begin{minipage}{0.9\linewidth}
\textbf{Multiple coverage optimization:}

Find the minimal set of sensors $P = \{x_1,...,x_n\}$ such that
\begin{equation}\label{eqn:problem_statement}
  \text{Vol} \left( \left\{y\in \Omega \vert~ \mathcal{O}_\text{vis}(y; P) \geq k \right\} \right) \le (1 - \delta) \text{Vol}(\Omega_{\text{free}}),  
\end{equation}
for some given threshold $\delta \in (0,1).$
\end{minipage}}\\

While not explicitly stated, we prefer solutions where the sensors are spread out and do not occupy the same locations. This makes the sensor network more robust against adversarial attacks in practice.


\section{Methods}\label{sec:methods}

The optimization problems, such as the one defined in 
\eqref{eqn:problem_statement} has a daunting complexity. It was shown to be NP-complete \cite{URRUTIA2000973,1056648,1057165}. Assuming we want to place $K$ sensors on an $M^3$ grid over the environment $\Omega \subseteq \mathbb{R}^3$ by a brute-force search, we would have to check $M^{3K}$ sets of sensors. It is therefore necessary to relax the problem into a more tractable one. 

In the following exposition, we shall describe the proposed mathematical formulations on the continuum, since the notation is cleaner. 
In practice, we work with Cartesian grids. This means that sensors will be placed on the grid nodes, and the visibility information computed on the grid, using algorithms introduced in \cite{tsai2004visibility, visalg}, and integrals are approximated by simple Riemann sums.

\subsection{The greedy algorithm}
A popular approach is the greedy algorithm. Instead of looking for the optimal set of sensors, we intend to generate a sequence of sensor locations. A new location in the sequence is chosen via maximization of a predetermined gain function $\mathcal{G}_k$, depending on the existing (given) list of sensors:
$$\mathbf{x}_{m+1} \in \argmax_{\mathbf{x} \in \Omega_{\text{free}}} \mathcal{G}_k(\mathbf{x}; P),~~~~P=\{\mathbf{x}_1, \dots,\mathbf{x}_m\}.$$
In the next greedy step, $\mathbf{x}_{m+1}$ is added to the list $P$.
This approach leads to the following algorithm.
\begin{algorithm}[H]
\caption{The greedy algorithm}\label{alg:eps-greedy}
\begin{algorithmic}
\State $P = \text{list}()$
\While{$\int_{\Omega_{\text{free}}} \mathcal{O}_\text{vis}(\mathbf{y}; P) d\mathbf{y} < \delta k \lambda(\Omega_{\text{free}})$}
\State calculate $\mathcal{G}_k(\mathbf{x}, P)$ for every $\mathbf{x} \in \Omega_{\text{free}}$
\State choose $\mathbf{x}^* \in \{\mathbf{y} \in \Omega_{\text{free}}~\vert~ \mathcal{G}_k(\mathbf{y}, P) \geq (1-\epsilon)\max_{\mathbf{z} \in \Omega_{\text{free}}} \mathcal{G}_k(\mathbf{z}, P)\}$ randomly
\State P.append($\mathbf{x}^*$)
\EndWhile
\end{algorithmic}
\end{algorithm}

In each greedy step, 
we optimize a weighted sum of the volume having different minimal orders of visibility:
\begin{equation}\label{eq:f_k}
    f_k(P) = \sum_{i=1}^k w_i \int_{\Omega_\text{free}} \mathbbm{1}_{\{\mathcal{O}_{\text{vis}}(\mathbf{z}, P) \geq i\}} d\mathbf{z}.
\end{equation}
Each integral evaluates the volume of the sub-level sets of $\mathcal{O}_{\text{vis}}(\mathbf{z},P)$. We shall refer to $f_k(P)$ as the coverage of the set of sensors listed in $P$.

Under this optimization objective, the gain function $\mathcal{G}_k(\mathbf{x},P)$ describes in the increase in the coverage:
\begin{equation}\label{eq:gain_in_fk}
    f_k(P\cup\{\mathbf{x}\})=f_k(P)+\mathcal{G}_k(\mathbf{x},P).
\end{equation}
In the next Subsection, we present a theory relating the
efficiency of the greedy algorithm using \eqref{eq:gain_in_fk}.

\subsection{A submodularity theory}
In this Section, we will analyze our choice of gain function and provide proof of its efficiency. 



\begin{prop}\label{prop:monotone} (Monotonicity of $f_k$)
     If $w_i \geq 0$ for all $1\le i\le k$ in \eqref{eq:f_k}, then 
     $$f_k(A) \leq f_k(B),$$
      for all countable sets $A \subseteq B \subseteq \Omega_\text{free}.$ 
\end{prop}


\begin{prop}(Submodularity of $f_k$)\label{prop:submodular}
    If $w_{i+1} \leq w_i$ for all $1\le i \le k-1 $,  then
    $$\mathcal{G}_k(\mathbf{z}, A) \geq \mathcal{G}_k(\mathbf{z}, B),$$
    for any countable sets $A \subseteq B \subseteq \Omega_\text{free}$ and any $\mathbf{z} \in \Omega_\text{free}$.
\end{prop}
Note that the condition $w_{i+1} \leq w_i$ incentivizes the algorithm to produce spaced out sensors. Indeed, the weight $w_i$ can be interpreted as the gain when promoting a region of volume $1$ from order $i-1$ visibility to order $i$ visibility. Thus, the algorithm prefers lower-order visibility, which produces spaced-out sensor networks.
\begin{thm}
    Let $P_n = \{x_1,...,x_n\}$ be the set of $n$ sensors placed according to Algorithm~\ref{alg:eps-greedy} using $\mathcal{G}_k$ defined in \eqref{eq:f_k}-\eqref{eq:gain_in_fk} and parameter $\epsilon \in [0,1)$. 
    $$P_l^* = \{x_1^*,...,x_l^*\}\in\argmax_{P \subseteq \Omega_\text{free}, \vert P \vert = l} f_k(P)
    \implies f_k(P_n) \geq \left( 1 -  e^{-(1-\epsilon)\frac{n}{l}} \right) f_K(P_l^*).$$
\end{thm}
Because $f_k$ is monotone, adding new sensors will always improve the sensor network. The theorem above gives us a measure of quality depending on the number of sensors placed by the greedy algorithm compared to a fixed number of optimally placed sensors. Picking $l = n$ gives us a bound about the quality of the placed sensor network compared to optimally placed sensors of the same amount. Below, we have listed this efficiency for commonly used values of $\epsilon$.\\
\begin{center}
\begin{tabular}{|c|r|r|r|}
\hline
 $\epsilon$ & $0$ & $0.01$ & $0.05$ \\
\hline
 $~{f_k(P_n)}/{f_k(P_l^*)}~$ &  $~0.6321$ & $~0.6284$ & $~0.6133$\\
\hline
\end{tabular}
\end{center}
This shows that the commonly used values for $\epsilon$ have little effect on the theoretical guarantee we derived.

\subsection{Parallelizing the greedy algorithm}
While the greedy algorithm is sequential by construction, we propose using an ensemble of greedy sequences in parallel. Each greedy sequence involves the gain function for a lower order of visibility. 
The algorithm can be summarized by the following recursion formula:
$$\mathbf{x}_{m+1}^{\ell} \in \argmax_{\mathbf{x} \in \Omega_{\text{free}}} \mathcal{G}_1(\mathbf{x}; \mathbf{x}_1^\ell, \dots,\mathbf{x}_m^\ell),~~~\ell=1,2,\cdots, k.$$
It is made precise in Algorithm~\ref{alg:round_robin} below.


\begin{algorithm} 
\caption{The parallel greedy algorithm}\label{alg:round_robin}
\begin{algorithmic}
\State $P_1,...,P_k = \text{list}()$

\For{$i = 1:k$}
    \While{$\int_{\Omega_{\text{free}}} \min \left\{\mathcal{O}_\text{vis}(y; P_i), 1 \right\}dy < \frac{\delta}{k} \lambda(\Omega_{\text{free}})$}
        \If{len($P_i$) = 0}
        \State Pick $x \in \Omega_\text{free}$ randomly
        \State $P_i$.append($x$)
        \Else
        \State $\mathcal{G}_1(\mathbf{x}, P_i)$ for every $x \in \Omega_{\text{free}}$
        \State compute $M = \max_{z \in \Omega_{\text{free}}} \mathcal{G}_1(\mathbf{z}, P_i)$
        \State choose $\mathbf{x}^* \in \{\mathbf{y} \in \Omega_{\text{free}}~\vert~ \mathcal{G}_1(\mathbf{y}, P_i) \geq (1-\epsilon) M\}$ randomly
        \State $P_i$.append($\mathbf{x}^*$)
        \EndIf
    \EndWhile
\EndFor
\State return $\bigcup_{i=1}^k P_i$
\end{algorithmic}
\end{algorithm}

The main advantage of this algorithm is that the separate computations of the single coverage problems (the use of $\mathcal{G}_1$) inside the for loop are independent and can, therefore, be performed in parallel. Further, in \cite{ly2018greedy}, it was shown that $\mathcal{G}_1$ can be efficiently approximated using a neural network with UNet architecture~\cite{ronneberger2015unet}. 
Performing these calls in parallel may optimize GPU utilization to its full potential capacity.

The algorithm generates $k$ sequences of sensor locations independently until the end of the parallel run. To solve the problem given by Equation~\eqref{eqn:problem_statement}, we need to ensure that the merged sensor set covers a volume of $(1-\delta) \lambda(\Omega_\text{free})$ with order $k$. In the worst case, the regions not entirely observed by the separate sensor sequences are completely disjoint. 
Accordingly, we prescribe a stricter termination criterion of covering a volume of $\left(1-\frac{\delta}{k} \right) \lambda(\Omega_\text{free})$.

\subsection{Learning the gain function}

The evaluation of the gain function can be costly, especially for large maps. Because we use the gain function to choose the next sensor location as close to the maximum, evaluating the gain function $\mathcal{G}_k$ at every possible point of sensor placement is necessary. Additionally, evaluating $\mathcal{G}_k$ is very costly since the visibility function is the solution of a global partial differential equation. Assume we want to place $K$ sensors on an $M \times M$ grid in $\Omega_{\text{free}} \subset \mathbb{R}^2$ to achieve target order of visibility $k$. The algorithm described in \cite{tsai2004visibility} has a complexity of $\mathcal{O}(M^2)$ for the computation of the visibility function. For the evaluation of $\mathcal{G}_k$ at one point, we need to compute visibility once, update $\mathcal{O}_\text{vis}$ and compute volumes. Updating $\mathcal{O}_\text{vis}$ has complexity $\mathcal{O}(M^2k\log k)$ and computing the volumes is $\mathcal{O}(M^2k)$. This yields a total computational complexity of 
$$\mathcal{O}(KM^2k\log k)$$

To speed up the evaluation of $\mathcal{G}_k$, we propose a deep learning strategy in which a neural network is trained to predict the gain function at every possible sensor location for a class of environments.
Once trained, this allows for a quick evaluation of the gain function and, therefore, a considerable speed-up in applying the greedy algorithm. 
From an initial sensor location, the greedy algorithm defines a discrete dynamical system, parameterized by the environment $\Omega_{free}$, that yields a sequence of ``images'' $\mathcal{G}_k,$ $k=0,1,2,...$. 
It is essential that the training data set adequately samples the causality defined by the greedy algorithm.

\paragraph{The neural network.} 
We chose the TiraFL architecture described in \cite{jégou2017layers} to learn $\mathcal{G}_k$ from suitable training examples. In these examples, we assume that the obstacles are described using the graph of a function $h_\text{obs}$:
$$\Omega_\text{free} = \{ (x,y,z)\in\Omega: h_\text{obs}(x,y) > z\}.$$
Then for $\mathbf{x} \in \Omega_\text{free}$, $\phi_{\mathbf{x}}$ can also be described by a function $g_{\mathbf{x}}$: 
$$\phi_{\mathbf{x}}(\mathbf{z}) = 1 \iff g_{\mathbf{x}}(z_1,z_2) \leq z_3,~~~\mathbf{z}=(z_1, z_2, z_3).$$
Correspondingly, the sensor n-coverage can be described by
$$\Psi_n(\mathbf{z}, P) = \nmin_{\mathbf{x} \in P} g_{\mathbf{x}}(\mathbf{z}),$$
where 
$\nmin_{\mathbf{x} \in P} g_\mathbf{x} = \max \left\{ \lambda \in \left\{ g_\mathbf{x} \text{: } \mathbf{x} \in P \right\} \text{: } \vert \left\{ g_\mathbf{x} \text{: } \mathbf{x} \in P \right\} \cap (-\infty, \lambda) \vert < n \right\}$\\
returns the n-th smallest value in $\{g_\mathbf{x}(z_1, z_2) \text{: } \mathbf{x} \in P\}$. For example, suppose $\left\{ g_\mathbf{x} \text{: } \mathbf{x} \in P \right\} = \{42, 4, 1337, 69\}$ then
$$\onemin_{\mathbf{x} \in P} g_{\mathbf{x}}(\mathbf{z}) = 4, \quad \twomin_{\mathbf{x} \in P} g_{\mathbf{x}}(\mathbf{z}) = 42, \quad \threemin_{\mathbf{x} \in P} g_{\mathbf{x}}(\mathbf{z}) = 69.$$

We use a neural network to approximate $\mathcal{G}_k$ in the following fashion:
$$\mathcal{G}_k^\theta (h_\text{obs}, \Psi_1(\cdot, P),...,\Psi_k(\cdot, P)) \approx 
    [ \mathcal{G}_k(\cdot, P),
    V_1(\cdot, P),
    \cdots,
    V_k(\cdot, P)]^T,
$$
where 
\begin{equation}\label{eqn:Vi}
  V_i(\mathbf{x}, P) = \int_{\Omega_\text{free}} \left( \mathbbm{1}_{\{\mathcal{O}_{\text{vis}}(\mathbf{z}, P \cup \{\mathbf{x}\}) \geq i\}} - \mathbbm{1}_{\{\mathcal{O}_{\text{vis}}(\mathbf{z}, P) \geq i\}} \right) d\mathbf{z}.  
\end{equation}
For the training, we use the Adam optimization algorithm.

\paragraph{The training data sets}\label{sec:data}
The training examples are generated by running Algorithm~\ref{alg:eps-greedy} on an environment from a data set. We use random crops of Massachusetts building footprints from \cite{MnihThesis}. Then, using the flood-fill algorithm, we distinguish the separate buildings and assign them a random height in $(0,1)$, making it a 3D environment. 

For these environments, we run the greedy algorithm with target visibility order $k = 3$ until 99\% of $\Omega_{\text{free}}$ has visibility of order 3. In every iteration of the algorithm, we store the tensor $\Psi$ as described above and the matrix containing the height map of the environment as the input. 


As labels for the data set, we use a matrix containing evaluations of the gain function using $w_1 = 1, w_2 = \frac{1}{2}, w_3 = \frac{1}{4}$ at every grid point as well as matrices containing evaluations of $V_i(x_1, x_2, P)$ as defined in Equation \eqref{eqn:Vi} at every grid point. Storing these in addition to the gain function allows for different weights $w_1,...,w_3$ without retraining the neural network. 

In the following, we look at three different data sets.
\begin{center}
\begin{tabular}{| c | r | r | r |}
\hline
    &$D_0$ & $D_\epsilon$ & $D_+$\\
\hline
    \# Data Points & ~25,202 & ~11,924 & ~24,697 \\
    $\epsilon$ & $0$ & $0.05$ & $~0.0$ and $0.05$\\
\hline
\end{tabular}
\end{center}
In the above, $D_+$ consists of $D_\epsilon$ and 12,773 data points from $D_0$.  A link to the data sets will be shared if the paper is accepted. 

We shall use the notation $\mathcal{G}^\theta_k[D]$ to denote a network trained with the data set $D$. In the next Section, we will compare the performance of the greedy algorithm using $\mathcal{G}^\theta_k[D_0],$ $\mathcal{G}^\theta_k[D_\epsilon],$ and $\mathcal{G}^\theta_k[D_+].$



\section{Numerical experiments}
In this Section, we present some numerical experiments using the proposed algorithms and compare their performance. 
We used environments from the same distribution described in Section~\ref{sec:data} in the experiments.

We apply Algorithm~\ref{alg:eps-greedy} and Algorithm~\ref{alg:round_robin} to the environment under the assumption that we are only allowed to place sensors in the streets (at height 0) of the environment. This is closely related to the practical problem of computing optimal surveillance camera positions in an urban environment.

We will compare the performance of Algorithm~\ref{alg:eps-greedy} using different approximations to $\mathcal{G}_k$ and Algorithm~\ref{alg:round_robin} using $\mathcal{G}_1$. 

In the numerical experiments, we use a $128\times128$ uniform Cartesian grid over $\Omega$. On this grid, a brute force computation of $\mathcal{G}_k$ by optimized C++ code took (represenatively) $1.44 s$ on an Intel Xeon Gold 6248R CPU where the evaluation of the gain function is performed in parallel on 24 cores. In comparison the feedforward evaluation of the network took $0.08 s$ on a Nvidia V100 GPU.



\subsection{Results from Algorithm~\ref{alg:eps-greedy} using $\mathcal{G}_k$}
\begin{figure}
    \centering
\includegraphics[width = 0.9\linewidth]{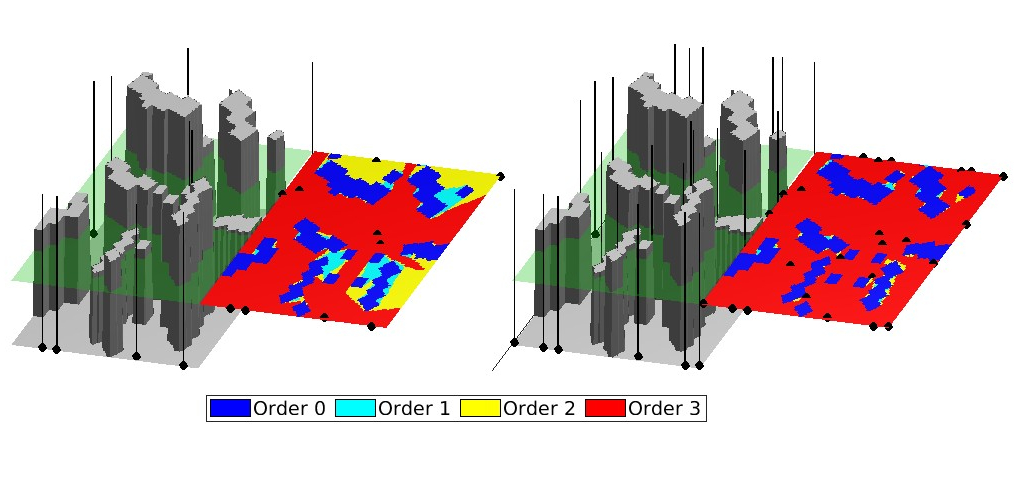}
    \caption{Algorithm~\ref{alg:eps-greedy}. Left: 10 sensors places, 45\% order 3 coverage. Right: 24 sensors places, 90\% order 3 coverage.}
    \label{fig:greedy_result}
\end{figure}

As an initial example, we use the three-dimensional environment shown in Figure~\ref{fig:obstacle}. We then use  Algorithm~\ref{alg:eps-greedy} (with brute force computation of $\mathcal{G}_k$ to compute the sequence of optimal sensor locations using $\epsilon = 0.01$ and $\mathcal{G}_k$ defined in \eqref{eq:f_k}-\eqref{eq:gain_in_fk} as a gain function. We chose the target order of visibility $k = 3$, and the algorithm terminates when 90\% of the free space is observed with order 3. Two snapshots of the simulation are demonstrated in Fig.~\ref{fig:greedy_result}.
In this Figure, as well as the subsequent Figures, the black lines further enhance the visibility of the black dots, making sure that they are visible even if hidden behind parts of the environment. The plane to the right of it shows the values of $\mathcal{O}_\text{vis}$ of the slice indicated by the green plane.

From figure~\ref{fig:greedy_result}, we observe that with 10 sensors, the majority of the free space in the slice is already observed. However, only with order 1 visibility. To cover 90\% of the free space with order 3 visibility, it was necessary to place 24 sensors, as shown in the right figure.




\subsection{Results from Algorithm~\ref{alg:round_robin} using $\mathcal{G}_1$ on three parallel runs}

Figure~\ref{fig:roundrobin_results} shows two stages of Algorithm~\ref{alg:round_robin} using $\mathcal{G}_1$ and 3 parallel runs. The left figure shows the order of the visibility map after 9 sensors, i.e., 3 from each separate run, have been placed. We see that the map has already been covered with order 1 visibility, and a large portion is covered with order 3 visibility. The right figure shows the terminal stage of the algorithm. 18 sensors were needed to cover 90\% of the free space with order 3 visibility. 

\begin{figure}
    \centering
    \includegraphics[width = 0.9\linewidth]{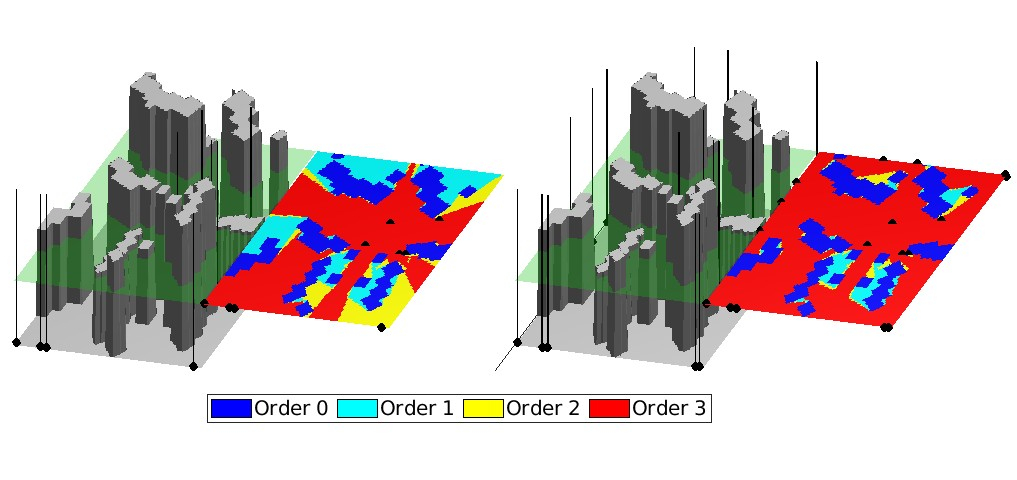}
    \caption{Algorithm~\ref{alg:round_robin} Left: 9 sensors placed, 43\% order 3 coverage. Right: 18 sensors places, 90\% order 3 coverage.}
    \label{fig:roundrobin_results}
\end{figure}

\subsection{Results from Algorithm~\ref{alg:eps-greedy} using the learned gain functions}\label{sec:net_results}

We apply Algorithm~\ref{alg:eps-greedy} using gain functions defined by neural networks described in Section~\ref{sec:methods}.
Two snapshots of a simulation using 
$\mathcal{G}^\theta_k[D_+]$, a network trained with the data set $D_+$ and $\epsilon=0.01$, are presented 
in Fig.~\ref{fig:net_res}.






Our extensive experiments show that the neural network trained on $D_0$ performs poorly compared to the one trained on $D_+$, even though the two data sets have roughly the same amount of data points (see Section~\ref{sec:data}). 

To explain the possible cause, we analyze the geometric properties of the data sets using the coordinate frame constructed by the
principal value analysis of $D_+$.
In Figure~\ref{fig:geometry}, we compare the variances of the projections of X points uniformly sampled from each of $D_0$, $D_\epsilon$, and $D_+$ onto this coordinate frame.
Notice that the distribution of $D_+$ is significantly fuller than $D_0$. We think this may cause the performance issue if one trained a network with only $D_0$ -- when the network makes inferences for inputs far from the center of the distribution.


\begin{figure} 
    \centering
    \includegraphics[width = 0.9\linewidth]{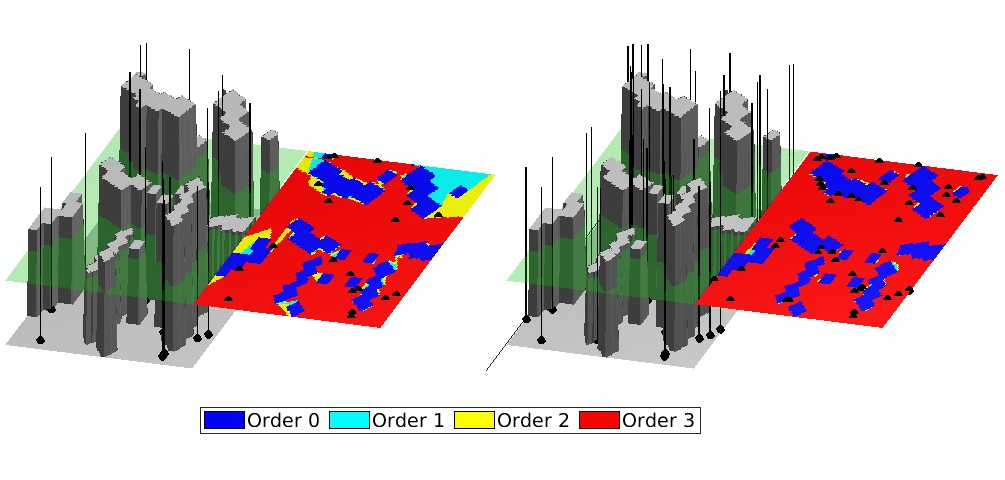}
\caption{Algorithm~\ref{alg:eps-greedy} using $\mathcal{G}^\theta_k[D_+]$.  Left: 20 sensors placed, 57\% order 3 coverage. Right: 47 sensors placed, 90\% order 3 coverage.}
    \label{fig:net_res}
\end{figure}

\begin{figure} 
    \centering
    \includegraphics[width = 0.6\linewidth]{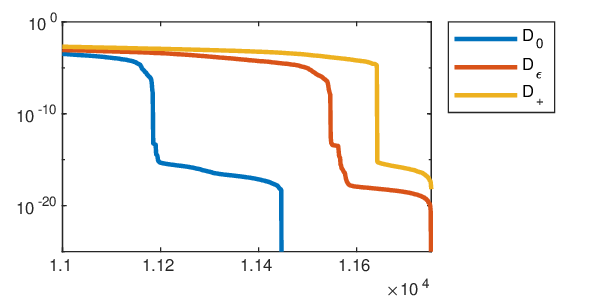}
    \caption{Ordered values of the singular values starting from the 11,000-th singular value.}
    \label{fig:geometry}
\end{figure}

\begin{figure}
    \centering
    \includegraphics[width = 0.5\linewidth]{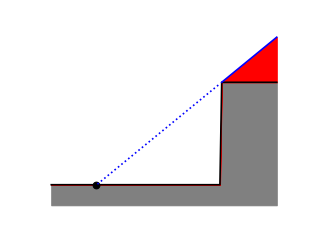}
    \caption{Urban environment. 
    The red area is impossible to observe from the ground. }
    \label{fig:rooftop}
\end{figure}

\subsection{Performance statistics and comparisons}

To evaluate the algorithms' performance, we conducted a statistical analysis of the results by randomly sampling 50 environments from the same distribution described in Section~\ref{sec:data}. As depicted in Figures~\ref{fig:greedy_result}-\ref{fig:net_res}, the volume of fully observed regions increases quickly in the initial stages of the algorithms and then stagnates. To investigate this further, we analyzed the number of sensors required for the percentage of fully observed free space to reach a fixed threshold with a maximum of 200 sensors placed. In Figures~\ref{fig:box}~and~\ref{fig:box_70}, the results are presented in box-plots indicating the 25\%, 50\%, and 75\% quantiles, illustrating the distribution of the needed number of sensors across the different sample environments.

In our analysis, we examined the behavior of various algorithms at different stages of iterations in their simulations. We used a threshold of 70\% in Figure~\ref{fig:box_70} to investigate the algorithms' performance in the earlier stages of simulations. We observed no significant performance gap, even with the random sensor placement. 

In Figure~\ref{fig:box}, we increased the threshold to 90\%, corresponding to the algorithm's late stages. The random point placement became significantly less effective. Not surprisingly, Algorithm~\ref{alg:eps-greedy} and its neural network approximation 
had more consistent performance and reached the threshold with less than 100 sensors most of the time. Furthermore, Algorithm~\ref{alg:round_robin} seemed as competitive as Algorithm~\ref{alg:eps-greedy}!

We are also interested in understanding how quickly the observed volume increases as one adds more sensors at locations prescribed by the algorithms. We show in Figure~\ref{fig:scatter} scatter plots of the number of sensors versus the achieved percentage of full (multiple) coverage within the 50 sample maps. Figure~\ref{fig:scatter} shows that Algorithm~\ref{alg:eps-greedy}, its neural network approximation as well as Algorithm~\ref{alg:round_robin} achieve the threshold more consistently and with fewer sensors. 

\paragraph{End-stage performance and premature saturation of sensors' coverage.}
In our implementation, the algorithms do not terminate when the coverage increment becomes small. 
Therefore, the algorithm may continually add new sensors when the sensors' coverage saturates before reaching the preset termination coverage. See, in Figure~\ref{fig:scatter}, the horizontal blue lines around the threshold coverage.
This premature saturation is possible because, in our simulations, we restrict the sensors to the ground level (even though the visibility and gain computations are performed in three dimensions). See Figure~\ref{fig:rooftop}. 
For this reason, we suspect that our greedy algorithms' actual statistical performance is better. 
We leave this for a future investigation.

\section{Conclusion}
This paper presented a greedy (next-best-view) formulation for optimal multiple-view surveillance of complex environments. The optimization refers to the objective of using the fewest sensors for such tasks. 
We showed that the proposed relaxed form of the optimization problem can be solved efficiently and derived theoretical lower bounds for efficiency. We introduced a Deep Learning strategy to accelerate the computation of the greedy algorithm. We discovered that minor differences in the training data set's distribution could result in significant differences in the network's inference performance. 

The proposed parallel greedy algorithm, presented as Algorithm~\ref{alg:round_robin}, showed very promising performance. It may provide a new way of designing ``next-best-view'' algorithms for a group of agents in robotic applications. We plan to investigate the analytical properties of this approach in the future.





\begin{figure}
    \centering
    \includegraphics[width = 0.8\linewidth]{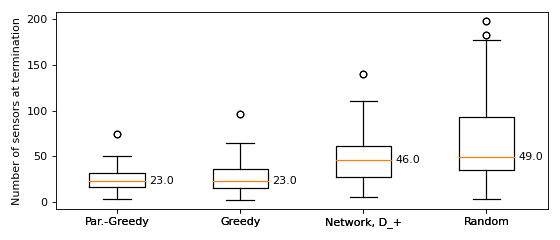}
    \caption{Box-plot of the number of sensors needed to achieve 90\% order 3 coverage. 
    }
    \label{fig:box}
\end{figure}

\begin{figure}
    \centering
    \includegraphics[width = 0.8\linewidth]{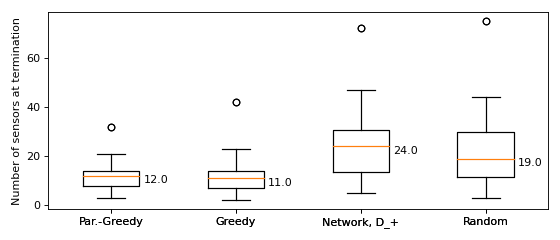}
    \caption{Box-plot of the number of sensors needed to achieve 70\% order 3 coverage. 
    }
    \label{fig:box_70}
\end{figure}

\begin{figure}
    \centering
    \includegraphics[width = \linewidth]{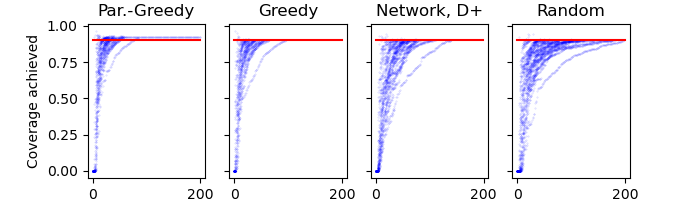}
    \caption{The number of sensors needed for order 3 visibility in the simulations run over 50 different sample maps. For all of the methods shown, we chose $\epsilon = 0.01$.}
    \label{fig:scatter}
\end{figure}

\section{Acknowledgements}
Taus is supported by Army Research Office, under Cooperative Agreement Number W911NF- 19-2-0333. Tsai is partially supported by Army Research Office, under Cooperative Agreement Number W911NF-19-2-0333 and National Science Foundation Grant DMS-2110895.

\appendix

\section{Proofs}

In this section we provide mathematical proofs for the statements made in the main paper. We also provide additional statements that help clarify the reasoning in the paper.

\begin{prop}(Monotonicity of $f_k$)
     If $w_i \geq 0$ for all $1\le i\le k$, then 
     $$f_k(A) \leq f_k(B),$$
      for all countable sets $A \subseteq B \subseteq \Omega_\text{free}.$ 
\end{prop}

\begin{proof}
    Suppose $A \subseteq \Omega_\text{free}$ and $x,z \in \Omega_\text{free}$. Then
    $$\mathcal{O}_\text{vis}(z, A \cup \{ x \}) = \sum_{y \in A \cup \{x\}} \phi_y(z) \geq \sum_{y \in A} \phi_y(z) = \mathcal{O}_\text{vis}(z, A).$$
    Since $\phi_y \geq 0$ as it is an indicator function.
    Thus for any $i \in \{1,...,k\}$ $\mathcal{O}_\text{vis}(z, A) \geq i$ implies $\mathcal{O}^k_\text{vis}(z, A \cup \{ x \}) \geq i$ and therefore
    $$ \mathbbm{1}_{\{\mathcal{O}_{\text{vis}}(z, A \cup \{ x \}) \geq i\}} \geq \mathbbm{1}_{\{\mathcal{O}_{\text{vis}}(z, A) \geq i\}}.$$
    By monotonicity of integrals and the fact that all $w_i \geq 0$ this shows that
    $$f_k(A \cup \{ x \}) \geq f_k(A)$$
    and by induction therefore also
    $$f_k(A) \leq f_k(B)$$
    for any countable set $B \supseteq A$.
\end{proof}

\begin{prop}(Submodularity of $f_k$)
    If $w_{i+1} \leq w_i$ for all $1\le i \le k-1 $,  then
    $$\mathcal{G}_k(\mathbf{z}, A) \geq \mathcal{G}_k(\mathbf{z}, B),$$
    for any countable sets $A \subseteq B \subseteq \Omega_\text{free}$ and any $\mathbf{z} \in \Omega_\text{free}$.
\end{prop}

\begin{proof}
    Let $P$ be an arbitrary countable subset of $\Omega_\text{free}$. Then 
    $$\int_{\Omega_\text{free}} \mathbbm{1}_{\{\mathcal{O}_{\text{vis}}(z, P \cup \{ x \}) \geq i\}} dz = \int_{\Omega_\text{free}} (1 - \phi_x(z) + \phi_x(z))\mathbbm{1}_{\{\mathcal{O}_{\text{vis}}(z, P \cup \{ x \}) \geq i\}} dz.$$
    Splitting up the integral and treating the terms separately we see that
    $$\int_{\Omega_\text{free}} (1 - \phi_x(z))\mathbbm{1}_{\{\mathcal{O}_{\text{vis}}(z, P \cup \{ x \}) \geq i\}} dz = \int_{\{z \in \Omega_\text{free} \text{: } \phi_x(z) = 0\}} \mathbbm{1}_{\{\mathcal{O}_{\text{vis}}(z, P \cup \{ x \}) \geq i\}} dz$$
    and
    $$\int_{\Omega_\text{free}} \phi_x(z)\mathbbm{1}_{\{\mathcal{O}_{\text{vis}}(z, P \cup \{ x \}) \geq i\}} dz = \int_{\{z \in \Omega_\text{free} \text{: } \phi_x(z) = 1\}} \mathbbm{1}_{\{\mathcal{O}_{\text{vis}}(z, P \cup \{ x \}) \geq i\}} dz$$
    On the sub-domain where $\phi_x \equiv 0$
    $$\mathcal{O}_\text{vis}(z, P\cup \{ x \}) = \mathcal{O}_\text{vis}(z, P).$$
    On the other hand on the sub-domain where $\phi_x \equiv 1$
    $$\mathcal{O}_\text{vis}(z, P\cup \{ x \}) = 1 + \sum_{y \in P} \phi_y(z)$$
    and thus for $j \in \{1,...,k\}$
    $$\mathcal{O}_\text{vis}(z, P\cup \{ x \}) \geq j \iff \mathcal{O}_\text{vis}(z, P) \geq j-1.$$
    Putting the terms back together then yields
    $$\int_{\Omega_\text{free}} \mathbbm{1}_{\{\mathcal{O}_{\text{vis}}(z, P \cup \{ x \}) \geq i\}} dz = \int_{\Omega_\text{free}} (1 - \phi_x(z))  \mathbbm{1}_{\{\mathcal{O}_{\text{vis}}(z, P) \geq i\}} + \phi_x(z)  \mathbbm{1}_{\{\mathcal{O}_{\text{vis}}(z, P) \geq i-1\}} dz$$
    which is equal to
    $$\int_{\Omega_\text{free}} \mathbbm{1}_{\{\mathcal{O}_{\text{vis}}(z, P) \geq i\}} +  \int_{\Omega_\text{free}} \phi_x(z)  \left[ \mathbbm{1}_{\{\mathcal{O}_{\text{vis}}(z, P) \geq i-1\}} -\mathbbm{1}_{\{\mathcal{O}_{\text{vis}}(z, P) \geq i\}} \right] dz.$$
    This shows that
    $$\Delta f(x, P) = \sum_{i=1}^k w_i \int_{\Omega_\text{free}} \phi_x(z)  \left[ \mathbbm{1}_{\{\mathcal{O}_{\text{vis}}(z, P) \geq i-1\}} -\mathbbm{1}_{\{\mathcal{O}_{\text{vis}}(z, P) \geq i-1\}} \right] dz.$$
    Note that $\mathbbm{1}_{\{\mathcal{O}_{\text{vis}}(z, P) \geq 0\}} \equiv 1$. Pulling the sum inside the integral and transforming the summation parameter then yields
    $$\Delta f(x, P) = \int_{\Omega_\text{free}} \phi_x(z) \left[ w_1 -  w_k \mathbbm{1}_{\{\mathcal{O}_{\text{vis}}(z, P) \geq k\}} + \sum_{i=1}^{k-1} (w_{i+1} - w_i) \mathbbm{1}_{\{\mathcal{O}_{\text{vis}}(z, P) \geq i\}}\right] dz.$$
    Note that this statement is true for any countable subset $P$ and is therefore also true for $A$ and $B$ respectively. Using this fact we can show that $\Delta f_k(x, A) - \Delta f_k(x, B)$ can be written as the sum of
    $$\int_{\Omega_\text{free}} \phi_x(z) w_k \left(\mathbbm{1}_{\{\mathcal{O}_{\text{vis}}(z, B) \geq i\}} - \mathbbm{1}_{\{\mathcal{O}_{\text{vis}}(z, A) \geq i\}} \right) dz$$
    and 
    $$\int_{\Omega_\text{free}} \phi_x(z) \left[ \sum_{i=1}^{k-1} (w_{i+1} - w_i) \left(\mathbbm{1}_{\{\mathcal{O}_{\text{vis}}(z, A) \geq i\}} - \mathbbm{1}_{\{\mathcal{O}_{\text{vis}}(z, B) \geq i\}} \right) \right] dz.$$
    In the proof of monotonicity it was already shown that 
    $$\mathbbm{1}_{\{\mathcal{O}_{\text{vis}}(z, B) \geq i\}} \geq \mathbbm{1}_{\{\mathcal{O}_{\text{vis}}(z, A) \geq i\}}.$$
    Further note that $\phi_x(z) \geq 0$ for all $z \in \Omega_\text{free}$. Therefore $w_k \geq 0$ and $w_{i+1} - w_i \leq 0$ are sufficient conditions for both of the above terms to be positive. This shows that if for all $i \in \{1,...,k\}$ $0 \leq w_{i+1} \leq w_i$ then
    $$\Delta f_k(x, A) \geq \Delta f_k(x, B)$$
\end{proof}

\begin{thm}
    Let $P_n = \{x_1,...,x_n\}$ be the set of $n$ sensors placed according to algorithm~1 using $\Delta f_k$ as the gain function and parameter $\epsilon \in [0,1)$. Further assume $P_l^* = \{x_1^*,...,x_l^*\}$ is the solution of
    $$\argmax_{P \subseteq \Omega_\text{free}, \vert P \vert = l} f_k(P).$$
    Then
    $$f_k(P_n) \geq \left( 1 -  e^{-(1-\epsilon)\frac{n}{l}} \right) f_K(P_l^*)$$
\end{thm}

\begin{proof}
    Let $m \in \mathbb{N}$ be arbitrary. 
    First since $f_k$ is monotone we observe that 
    $$f_k(P_l^*) \leq f_k(P_l^* \cup P_m).$$
    By definition of the discrete derivative
    $$\Delta f_k(x, P) = f_k(P \cup \{ x \}) - f_k(P)$$
    which yields
    $$f_k(P_l^* \cup P_m) = f_k(P_m) + \sum_{i=1}^l \Delta f_k(x_i^*, P_m \cup \{x_1^*,...,x_{i-1}^*\}).$$
    Since $f_k$ is submodular
    $$\Delta f_k(x_i^*, P_m \cup \{x_1^*,...,x_{i-1}^*\}) \leq \Delta f_k(x_i^*, P_m) \leq  \max_{x \in \Omega_{\text{free}}} \Delta f_k(x, P_m)$$
    for any $i \in \{1,...,l\}$. 
    Using the greedy algorithm we know that the sensor $x_{m+1}$ will be chosen such that the gain function satisfies
    $$\mathcal{G}_k(x_{m+1}, P_m) \geq (1 - \epsilon) \max_{x \in \Omega_{\text{free}}} \mathcal{G}_k(x, P_m).$$
    Since $\mathcal{G}_k = \Delta f_k$ this shows
    $$\Delta f_k(x_i^*, P_m) \leq \frac{1}{1-\epsilon} \Delta f_k(x_{m+1}, P_m) = \frac{1}{1-\epsilon} \left( f_k(P_{m+1}) - f_k(P_m)) \right).$$
    Collectively we have shown that
    $$f_k(P_l^*) - f_k(P_m) \leq \frac{l}{1-\epsilon} \left( f_k(P_{m+1}) - f_k(P_m)) \right).$$
    Defining
    $$\delta_m = f_k(P_l^*) - f_k(P_m), \delta_0 = f_k(P_l^*)$$\
    we see that
    $$\delta_m \leq \frac{l}{1 - \epsilon} (\delta_m - \delta_{m+1})$$
    or equivalently
    $$\delta_{m+1} \leq \left( 1 - \frac{1-\epsilon}{l} \right) \delta_m.$$
    Solving this recursion yields
    $$\delta_n \leq \left( 1 - \frac{1-\epsilon}{l} \right)^n \delta_0.$$
    Plugging in the definition of $\delta_n$ we see
    $$\left[ 1 - \left( 1 - \frac{1-\epsilon}{l} \right)^n \right] f_k(P_l^*) \leq f_k(P_n).$$
    Finally using the inequality $1 - x \leq e^{-x}$ shows
    $$\left[ 1 - e^{-(1-\epsilon) \frac{n}{l}} \right] f_k(P_l^*) \leq f_k(P_n).$$
\end{proof}

\section{Choice of hyper parameters for neural network}
As the network architecture we chose the TiraFL architecture as described in \cite{jégou2017layers}. Following the naming convention the table below shows the hyper parameters chosen for our application.\\

\begin{center}
\begin{tabular}{| l | l |}
    \hline
    in\_channels & 4 \\
    \hline
    down\_blocks & (4,5,7,7,12,12)\\
    \hline
    up\_blocks & (12,12,7,7,5,4)\\
    \hline
    bottleneck\_layers & 15\\
    \hline
    growth\_rate & 16\\
    \hline
    out\_chans\_first\_conv & 48\\
    \hline
    n\_classes & 4\\
    \hline
\end{tabular}
\end{center}


%
%
\bibliographystyle{splncs04}
\bibliography{main}
\nocite{*}

\end{document}